\documentclass[11pt]{article}
\usepackage[margin=1in]{geometry}
\usepackage{amsmath,amssymb,amsthm}
\usepackage{booktabs}
\usepackage{threeparttable}
\usepackage{adjustbox}
\usepackage{graphicx}
\usepackage{float}
\usepackage{placeins}
\usepackage[protrusion=true,expansion=false]{microtype}
\usepackage[round,authoryear]{natbib}
\usepackage{authblk}      
\usepackage{setspace}     
\usepackage[colorlinks=true,
            linkcolor=blue,
            citecolor=blue,
            urlcolor=blue]{hyperref}
 
\let\originalincludegraphics\includegraphics
\renewcommand{\includegraphics}[2][]{%
  \IfFileExists{#2}{\originalincludegraphics[#1]{#2}}{%
    \fbox{\parbox[c][2in][c]{0.88\linewidth}{\centering
      Figure file to be supplied: \texttt{\detokenize{#2}}}}}}
 
\newcommand{\R}{\mathbb{R}}
\newcommand{\E}{\mathbb{E}}

\newcommand{\norm}[1]{\left\lVert #1 \right\rVert}

\DeclareMathOperator*{\argmax}{arg\,max}

\newtheorem{assumption}{Assumption}

\newtheorem{proposition}{Proposition}
\newtheorem{theorem}{Theorem}

\newtheorem{remark}{Remark}
 
\title{\bfseries A Deep Learning Model for Spatially Clustered Data
via Differentiable Cluster Assignment}
 
\author[1]{Kexuan Li\thanks{kexuan.li.77@gmail.com}}
\author[2]{Weidong Ma\thanks{weidong.ma@pennmedicine.upenn.edu}}
\affil[1]{Bristol Myers Squibb}
\affil[2]{University of Pennsylvania, 
Department of Biostatistics}

\begin{document}

\maketitle


\begin{abstract}
We consider nonparametric regression when the association between a response
and its covariates changes across an unknown partition of a spatial domain.
The proposed estimator learns the partition and the cluster-specific
regression functions jointly.  A neural network depending only on location
determines cluster membership, while separate neural networks describe the
covariate--response relationship within the clusters.  An annealed softmax
relaxation permits gradient-based estimation of the otherwise discrete
assignments.  Graph-Laplacian and occupancy penalties are used to discourage
fragmented regions and degenerate solutions.  We establish identifiability up
to label permutation, bound partition error under a margin condition, and
decompose prediction risk into regression and assignment components.  The
resulting rate agrees with that of an oracle estimator when the partition is
estimated sufficiently accurately.  Simulations show that joint estimation
is useful when regression surfaces change abruptly across spatial boundaries,
including settings with nonlinear effects, unequal region sizes, preferential
sampling, and spatially correlated errors. Finally, a real data analysis is provided to demonstrate the validity
and effectiveness of the proposed method..
\end{abstract}


\clearpage
\section{Introduction}\label{sec:introduction}

Spatial heterogeneity in a regression relationship means that the association
between a response and its predictors changes with location.  This phenomenon
is not confined to one application area.  Spatial variation has been documented
in the marginal prices of housing attributes \citep{bitter2007housing,
bhattacharjee2016housing}, in relationships between environmental variables
\citep{mu2018spatial}, in geographic disease risk and exposure--outcome
associations \citep{wakefield2014geographic}, and in species--environment
relationships \citep{doser2024guidelines}.  Thus, an identical change in a
covariate need not have the same implication at two locations.  A global
regression compresses these local relationships into a single average and can
therefore conceal scientifically or substantively important variation
\citep{fotheringham2002geographically,comber2023route}.

Most established approaches represent such variation as a smooth function of
space.  Geographically weighted regression fits a locally weighted model at
each location \citep{brunsdon1996gwr,fotheringham2002geographically}; extensions
allow the coefficients to vary at different spatial scales
\citep{comber2023route}.  Spatially varying coefficient process models instead
place stochastic-process priors on coefficient surfaces
\citep{gelfand2003spatial}, while spline methods estimate smooth coefficient
functions over regular or irregular domains \citep{mu2018spatial}.  These
methods offer flexible descriptions of gradual spatial change.  Their basic
structural premise is nevertheless continuity or local smoothness, which is
different from a setting in which the regression relationship is relatively
homogeneous within regions but changes at their boundaries.

Region-based methods address the latter structure directly.  In housing, for
example, submarkets have been delineated from spatial variation in hedonic
relationships rather than imposed solely from administrative geography
\citep{bhattacharjee2016housing}.  More generally,
\citet{sugasawa2021spatially} estimate clustered linear regression coefficients
using a Potts-type penalty that favors common labels at neighboring sites.
\citet{lin2022spatially} allow coefficients to vary smoothly within a cluster
and to change abruptly between adjacent clusters, combining spline
approximation with a concave fusion penalty on a spatial graph.  These methods
show why spatial clustering should be linked to the regression relationship,
rather than obtained from geographic proximity alone.  They also leave a
substantive gap: region-specific effects are represented by linear coefficients
or smooth varying-coefficient expansions, whereas many applications require a
fully nonlinear response surface with interactions among covariates.

Neural networks provide one route to such nonlinear spatial regression, and a
small but growing statistical literature has developed around this idea.
\citet{zammitmangion2022deep} use deep compositions of spatial warpings to model
nonstationary and anisotropic covariance structure.  \citet{wang2023deep}
construct a neural architecture for nonlinear endogenous and covariate effects
in a spatial autoregressive model.  Particularly relevant here,
\citet{li2023semiparametric} use a sparsely connected ReLU network to estimate
a nonparametric mean function under Gaussian spatial dependence and derive a
convergence rate that reflects both network complexity and the strength of that
dependence.  These contributions establish the usefulness of deep learning for
flexible mean or dependence structures in spatial data.  Their targets differ
from the present one: they do not estimate an unknown discrete spatial
partition together with a separate nonlinear regression function for each
region.

We consider precisely this problem.  Cluster membership is determined by an
assignment function of spatial location, and the conditional mean in each
cluster is represented by a neural network of the nonspatial covariates.  The
location-only restriction preserves the geographic meaning of a cluster: two
observations at the same site cannot receive different regional labels merely
because their covariates differ.  At the same time, the within-region networks
allow nonlinear effects and interactions without prescribing their form.
Estimating the assignment and regression functions jointly lets changes in the
response--covariate relationship inform the boundary.  This distinguishes the
method from a two-stage procedure that first clusters locations according to
distance or sampling density and only then fits separate regressions.

The discrete assignment creates a computational obstacle.  Direct
optimization over all partitions is infeasible even for moderate sample sizes,
and a hard argmax operation does not admit ordinary gradient updates.  During
training we use the continuous categorical relaxations of
\citet{jang2017gumbel} and \citet{maddison2017concrete}, with a temperature that
decreases over epochs.  A graph-Laplacian penalty encourages neighboring
locations to have similar assignment probabilities, following the general
principle of graph-based regularization \citep{belkin2006manifold}.  A second,
weak occupancy penalty guards against fits in which nearly every observation
is assigned to the same cluster.  The final reported partition and predictions
use hard assignments; the relaxation is a device for estimation rather than a
change in the inferential target.

The theoretical analysis separates error from estimating the regression
functions from error caused by an estimated partition.  Existing approximation
and risk results show that ReLU networks can attain favorable rates over broad
nonparametric function classes \citep{yarotsky2017error,
schmidthieber2020nonparametric}, but they do not account for an unknown spatial
assignment.  We first establish identifiability, up to label permutation, from
overlap of the regional covariate distributions and separation of the
cluster-specific regression functions.  We then use a margin condition of the
type employed in plug-in classification theory
\citep{mammen1999smooth,audibert2007fast} to translate assignment-score error
into partition error.  A risk decomposition finally identifies conditions
under which estimating the partition does not alter the leading rate relative
to an oracle that knows the regions.

The numerical studies are organized around the distinction between smooth and
piecewise spatial heterogeneity.  The first experiment considers two-region
partitions with strong or weak separation and balanced or unequal region
sizes.  The second adds three curved regions, spatially correlated covariates,
preferential sampling, and spatially correlated errors; the latter setting is
outside the independent-sampling theory and is included as a robustness
assessment.  Across the nonlinear piecewise settings, the joint estimator is
substantially more accurate than a global neural network, geographically
weighted regression, and a two-stage spatial clustering procedure.
The King County housing analysis gives a complementary result: the estimated
three-region partition is stable and substantively interpretable, although a
global spatial neural network predicts better.  Reporting this contrast helps
clarify when an explicit partition is useful and when a smoothly varying
spatial model is preferable.

The remainder of the paper is arranged as follows.  Section~\ref{sec:model}
defines the model, and Section~\ref{sec:estimation} describes estimation and
selection of the number of regions.  Section~\ref{sec:theory} gives the main
theoretical results, with proofs in Appendix~\ref{app:proofs}.
Section~\ref{sec:simulation} reports the simulation studies, and
Section~\ref{sec:king-county} presents the housing application.  We conclude
in Section~\ref{sec:discussion}.
\section{Model}\label{sec:model}

Let \(S\subset\R^d\) denote the spatial domain. We observe independent
data
\[
    \{(s_i,x_i,y_i)\}_{i=1}^n,
\]
where \(s_i\in S\) is a spatial location, \(x_i\in\R^p\) is a vector of
covariates, and \(y_i\in\R\) is the response. The true number of spatial
clusters is denoted by \(K_0\). Let
\[
    C_0:S\to\{1,\ldots,K_0\}
\]
be the true cluster assignment function, and define the $k$-th true region as 
\[
    C_k=\{s\in S:C_0(s)=k\},
    \qquad k=1,\ldots,K_0 .
\]
The data-generating model is
\begin{equation}
    Y
    =
    f_{0,C_0(S)}(X)+\epsilon,
    \qquad
    \E(\epsilon\mid S,X)=0.
    \label{eq:model}
\end{equation}
Equivalently, if \(s_i\in C_k\), then
\[
    y_i=f_{0,k}(x_i)+\epsilon_i,
\]
where \(f_{0,k}:\R^p\to\R\) is the regression function for cluster
\(k\).

In this paper, we consider estimating the unknown function $f_{0,k}$ via a deep neural network owing to
the complexity of $f_{0,k}$ and the flexibility of neural networks. In addition, we consider to use another assignment network to estimate the probability of each loaction being assigned to region $C_k$. Specifically, the assignment network is
\[
    \Psi:S\to\Delta^{K-1},
    \qquad
    \Psi(s)=\{\Psi_1(s),\ldots,\Psi_K(s)\}^\top,
\]
where \(\Delta^{K-1}\) is the probability simplex. The assignment
network depends on spatial location only. The within-cluster regression
functions are represented by separate neural networks
\[
    \Phi_k:\R^p\to\R,
    \qquad k=1,\ldots,K.
\]
At inference, the fitted hard assignment is
\begin{equation}
    \widehat C(s)
    =
    \argmax_{1\le k\le K}\widehat\Psi_k(s),
    \label{eq:hard_assign}
\end{equation}
and the fitted response is
\begin{equation}
    \widehat f(s,x)
    =
    \widehat\Phi_{\widehat C(s)}(x).
    \label{eq:hard_pred}
\end{equation}

\section{Estimation and Optimization}\label{sec:estimation}

The hard assignment rule in \eqref{eq:hard_assign} is not differentiable
with respect to the parameters of \(\Psi\). During training, we therefore
replace the hard assignment by a continuous relaxation.  For temperature
\(\tau>0\), define
\begin{equation}
    \widetilde\pi_k(s)
    =
    \frac{
        \exp\{(\log \Psi_k(s)+g_k)/\tau\}
    }{
        \sum_{\ell=1}^K
        \exp\{(\log \Psi_\ell(s)+g_\ell)/\tau\}
    },
    \qquad
    g_k\stackrel{\mathrm{iid}}{\sim}\mathrm{Gumbel}(0,1).
    \label{eq:gumbel}
\end{equation}
The differentiable prediction used during training is
\begin{equation}
    \widetilde f(s_i,x_i)
    =
    \sum_{k=1}^K
    \widetilde\pi_k(s_i)\Phi_k(x_i).
    \label{eq:soft_pred}
\end{equation}
Equation~\eqref{eq:gumbel} is the Gumbel--softmax relaxation of
\citet{jang2017gumbel}, equivalently the Concrete relaxation of
\citet{maddison2017concrete}.  Setting \(g_k=0\) gives its deterministic
annealed-softmax version, which we use in the numerical experiments to avoid
adding Monte Carlo variation to the optimization.  In either case, as
\(\tau\to0\), the relaxed assignment concentrates on a vertex of the simplex;
with \(g_k=0\), that vertex corresponds to \(\argmax_k\Psi_k(s)\).

In practice, the temperature is annealed over training epochs \(t\):
\[
    \tau_t
    =
    \max\{\tau_{\min},\tau_0\rho^t\},
\]
where \(\tau_0>\tau_{\min}>0\) and \(0<\rho<1\). The hard assignment
rule \eqref{eq:hard_assign} is used after training.

Minimizing squared error alone may nevertheless lead to unstable or degenerate
solutions. For example, with a large number of clusters, the model may
create fragmented spatial regions or assign most observations to only
one cluster. We therefore minimize the penalized objective
\begin{equation}
\begin{aligned}
    L_n(\Psi,\Phi)
    &=
    \frac{1}{n}
    \sum_{i=1}^n
    \{y_i-\widetilde f(s_i,x_i)\}^2
    +
    \lambda_1 P_{\mathrm{sp}}(\Psi)
    +
    \lambda_2 P_{\mathrm{bal}}(\Psi).
\end{aligned}
\label{eq:loss}
\end{equation}
The spatial smoothness penalty is defined over a neighborhood graph on
the observed locations:
\begin{equation}
    P_{\mathrm{sp}}(\Psi)
    =
    \sum_{i\sim j}
    w_{ij}
    \norm{\Psi(s_i)-\Psi(s_j)}_2^2.
    \label{eq:spatial_penalty}
\end{equation}
Here \(i\sim j\) denotes spatial adjacency, for example \(k\)-nearest
neighbors or a distance-threshold graph, and \(w_{ij}\) is a nonnegative
weight that decreases with the distance between \(s_i\) and \(s_j\). The
penalty in \eqref{eq:spatial_penalty} is a graph Laplacian smoothness
penalty and discourages fragmented cluster boundaries. The balance penalty is based on the average fitted cluster occupancy
\[
    \bar\pi_k
    =
    \frac{1}{n}\sum_{i=1}^n \Psi_k(s_i),
    \qquad k=1,\ldots,K.
\]
We use
\begin{equation}
    P_{\mathrm{bal}}(\Psi)
    =
    \sum_{k=1}^K
    \bar\pi_k\log(K\bar\pi_k),
    \label{eq:balance_penalty}
\end{equation}
with the convention \(0\log0=0\). This is the Kullback--Leibler
divergence between the empirical cluster occupancy \(\bar\pi\) and the
uniform distribution on \(\{1,\ldots,K\}\). It is minimized when
\(\bar\pi_k=1/K\) for all \(k\), and therefore discourages cluster
collapse.

The assignment and regression networks are initialized jointly and trained by
mini-batch gradient descent.  At epoch \(t\), we set
\(\tau_t=\max\{\tau_{\min},\tau_0\rho^t\}\), evaluate the relaxed prediction in
\eqref{eq:soft_pred}, and update all network parameters using the objective in
\eqref{eq:loss}.  The deterministic relaxation sets \(g_{ik}=0\).  After the
last epoch, the relaxed assignments are discarded and the hard rules in
\eqref{eq:hard_assign} and \eqref{eq:hard_pred} are used.  Because the objective
is nonconvex, the numerical studies use more than one initialization and retain
the fit with the smallest prespecified selection loss. The number of fitted clusters \(K\) is treated as a tuning parameter.
In practice, \(K\) can be selected over a candidate set
\(\{2,\ldots,K_{\max}\}\) by cross-validated prediction error. As a
heuristic alternative, one may use a BIC-type criterion
\begin{equation}
    \mathrm{BIC}(K)
    =
    n\log\widehat\sigma_K^2
    +
    \mathrm{df}(K)\log n,
    \label{eq:bic}
\end{equation}
where \(\widehat\sigma_K^2\) is the residual variance under the fitted
model with \(K\) clusters and \(\mathrm{df}(K)\) is an effective
degrees-of-freedom approximation for the fitted networks. Since degrees
of freedom are not straightforward for neural networks, cross-validation
is the default recommendation.

\section{Theoretical properties}\label{sec:theory}

Throughout this section, we take \(K=K_0\).  Since cluster labels are
arbitrary, all results are stated up to a permutation of the labels.  Let
\[
  C_0:\mathcal S\to\{1,\ldots,K_0\}
\]
denote the true spatial assignment, and define
\[
  \mathcal C_k=\{s\in\mathcal S:C_0(s)=k\},
  \qquad k=1,\ldots,K_0.
\]
The data-generating model is
\begin{equation}\label{eq:true_model_theory}
  Y=f_{0,C_0(S)}(X)+\epsilon,
  \qquad
  E(\epsilon\mid S,X)=0.
\end{equation}

For any measurable assignment rule
\(q:\mathcal S\to\{1,\ldots,K_0\}\) and any collection of functions
\(\phi=(\phi_1,\ldots,\phi_{K_0})\), define the population risk
\begin{equation}\label{eq:population_risk}
  R(q,\phi)
  =
  E\left\{Y-\phi_{q(S)}(X)\right\}^2 .
\end{equation}
The oracle pair is \((C_0,f_0)\), where
\(f_0=(f_{0,1},\ldots,f_{0,K_0})\), and its risk is
\[
  R_0=R(C_0,f_0)=E(\epsilon^2).
\]
For each true cluster, define
\[
  \mu_k=P(S\in\mathcal C_k),
  \qquad
  \|g\|_{k,2}^2
  =
  E\{g(X)^2\mid S\in\mathcal C_k\}.
\]

\begin{assumption}[Nonvanishing regions]\label{ass:size}
There exists a constant \(\mu_0>0\) such that
\[
  \mu_k\geq \mu_0,
  \qquad k=1,\ldots,K_0.
\]
\end{assumption}

\begin{assumption}[Common covariate support]\label{ass:support}
There exists a probability measure \(\nu\) on \(\mathcal X\) such that, for
almost every \(s\), the conditional distribution of \(X\) given \(S=s\) has
density \(p_s\) with respect to \(\nu\).  There are constants
\(0<m_x<M_x<\infty\) such that
\[
  m_x\leq p_s(x)\leq M_x
\]
for \(\nu\)-almost every \(x\) and almost every \(s\).
\end{assumption}

\begin{assumption}[Separation]\label{ass:separation}
There exists a constant \(\kappa>0\) such that, for all \(k\neq l\),
\[
  \|f_{0,k}-f_{0,l}\|_{L^2(\nu)}\geq \kappa .
\]
\end{assumption}

\begin{assumption}[Regression regularity]\label{ass:regression}
For each \(k\),
\[
  f_{0,k}\in \mathcal H^\beta(\mathcal X),
\]
where \(\mathcal H^\beta(\mathcal X)\) is a H\"older ball on a compact set
\(\mathcal X\subset\mathbb R^p\).  In addition,
\[
  \sup_{1\leq k\leq K_0}\|f_{0,k}\|_\infty\leq M.
\]
The fitted regression functions are restricted to a uniformly bounded class
satisfying \(\|\phi_k\|_\infty\leq M\).  The error \(\epsilon\) is
sub-Gaussian.
\end{assumption}

We state the remaining conditions directly on the spatial boundary.  This is
preferable to treating a particular softmax score \(\Psi^*\) as a population
parameter: model~\eqref{eq:model} identifies its argmax partition, but many
different score scales and transformations produce the same partition.  Let
\[
  \mathcal B_0=\bigcup_{k=1}^{K_0}\partial\mathcal C_k,
  \qquad
  \mathcal B_0^t=\{s\in\mathcal S:d(s,\mathcal B_0)\leq t\}.
\]

\begin{assumption}[Boundary regularity]\label{ass:assignment_smooth}
Apart from finitely many junction points, \(\mathcal B_0\) can be covered by
finitely many coordinate neighborhoods in which it is the graph of a
\(\beta'\)-H\"older function.
\end{assumption}

\begin{assumption}[Boundary margin]\label{ass:margin}
There are constants \(c_0>0\), \(\gamma>0\), and \(t_0>0\) such that
\[
  P(S\in\mathcal B_0^t)\leq c_0t^\gamma,
  \qquad 0<t<t_0.
\]
\end{assumption}
For an estimated assignment \(q\), define the permutation-invariant
localization error
\begin{equation}\label{eq:boundary_localization}
  \delta_\partial(q,C_0)
  =
  \inf_{\sigma\in\mathfrak S_{K_0}}
  \inf\left\{t>0:
    \{s:\sigma(q(s))\neq C_0(s)\}\subseteq\mathcal B_0^t
  \right\}.
\end{equation}
Thus \(\delta_\partial\) measures how far an incorrectly assigned location
can lie from the true boundary; it does not compare nonidentified softmax
probabilities.

\begin{assumption}[Boundary localization rate]\label{ass:assignment_rate}
For a deterministic sequence \(a_n\to0\) and a constant \(C>0\),
\[
  E\{\delta_\partial(\widehat C,C_0)^\gamma\}
  \leq C a_n^\gamma .
\]
\end{assumption}

 Based on the above assumptions, the following theorem is claimed. 
\begin{theorem}[Identifiability]\label{thm:identifiability}
Suppose Assumptions~\ref{ass:size}--\ref{ass:separation} hold.  Among
measurable assignment rules
\(q:\mathcal S\to\{1,\ldots,K_0\}\) satisfying
\[
  P\{q(S)=a\}>0,
  \qquad a=1,\ldots,K_0,
\]
and square-integrable regression functions, the minimizers of \(R(q,\phi)\)
are exactly
\[
  q(s)=\sigma\{C_0(s)\},
  \qquad
  \phi_{\sigma(k)}=f_{0,k},
  \qquad k=1,\ldots,K_0,
\]
for some permutation \(\sigma\in\mathfrak S_{K_0}\), up to null sets.
\end{theorem}

Theorem~\ref{thm:identifiability} concerns the statistical target rather than
the training algorithm.  It shows that response-informed clustering has a
unique population interpretation after labels are ignored.  Both overlap and
separation are substantive: if either is removed, one fitted regional function may
represent more than one true region without increasing the risk.  The proof
is given in Appendix~\ref{app:proofs}.

\begin{theorem}[Partition error]\label{thm:partition}
Under Assumptions~\ref{ass:margin} and \ref{ass:assignment_rate},
\[
  E\left[
    \inf_{\sigma\in\mathfrak S_{K_0}}
    P_S\left\{
      \sigma\bigl(\widehat C(S)\bigr)\neq C_0(S)
      \mid \mathcal D_n
    \right\}
  \right]
  \lesssim a_n^{\gamma},
\]
where \(P_S(\cdot\mid\mathcal D_n)\) denotes probability over an independent
test location \(S\), conditional on the training sample \(\mathcal D_n\).
\end{theorem}

The conclusion has a direct geometric interpretation.  A boundary
displacement of order \(a_n\) can only affect locations in an
\(a_n\)-wide tube, to which the margin condition assigns probability of
order \(a_n^\gamma\). The proof is given in Appendix~\ref{app:proofs}.

\begin{remark}[Penalties and the population target]\label{rem:penalty}
The spatial smoothness and balance penalties are useful regularizers in
finite samples.  They are not part of the population risk
\eqref{eq:population_risk}.  This distinction matters for the balance
penalty, because it encourages nearly equal cluster sizes, whereas
Assumption~\ref{ass:size} only requires each true cluster to have positive
probability.  For the theoretical target to remain the unpenalized oracle
partition, the penalty weights should be viewed as tuning parameters that
vanish asymptotically, or one should verify that the penalties do not change
the population minimizer in the setting under study.
\end{remark}

\begin{proposition}[Oracle regression]\label{prop:oracle}
Let \(\widehat\Phi_k^{\rm or}\) be the least-squares neural-network
estimator of \(f_{0,k}\) trained using only observations with
\(S_i\in\mathcal C_k\).  Under Assumptions~\ref{ass:size},
\ref{ass:support}, and \ref{ass:regression},
\[
  E\|\widehat\Phi_k^{\rm or}-f_{0,k}\|_{k,2}^2
  \lesssim
  n^{-2\beta/(2\beta+p)}
\]
up to logarithmic factors.
\end{proposition}

This proposition is an oracle benchmark.  The cluster-size assumption makes
the within-region sample size proportional to \(n\), and common support makes
the relevant norms equivalent.  It does not involve boundary estimation;
the gap between this rate and the rate of the joint estimator therefore
isolates the cost of learning the spatial partition.  The proof is given in
Appendix~\ref{app:proofs}.

\begin{theorem}[Risk decomposition]\label{thm:risk}
Let
\[
  f_0(s,x)=f_{0,C_0(s)}(x),
  \qquad
  \widehat f(s,x)=\widehat\Phi_{\widehat C(s)}(x).
\]
Let \(\sigma^*\in\mathfrak S_{K_0}\) be a relabeling of the fitted clusters,
and define
\[
  \widetilde C(s)=\sigma^*\{\widehat C(s)\},
  \qquad
  \widetilde\Phi_k=\widehat\Phi_{\sigma^{*-1}(k)}.
\]
Then the fitted prediction is unchanged by relabeling:
\[
  \widehat f(s,x)
  =
  \widetilde\Phi_{\widetilde C(s)}(x).
\]
Under Assumption~\ref{ass:regression},
\[
  E\{\widehat f(S,X)-f_0(S,X)\}^2
  \leq
  \sum_{k=1}^{K_0}
  \mu_k E\|\widetilde\Phi_k-f_{0,k}\|_{k,2}^2
  +
  4M^2P\{\widetilde C(S)\neq C_0(S)\},
\]
where \((S,X)\) is an independent test point.
\end{theorem}

The first term is the error that remains when the estimated label is correct;
the second is the cost of assigning a location to the wrong regional function.  The
decomposition is modular and can accommodate regression estimators other than
neural networks.  Boundedness produces the constant \(4M^2\); alternative
moment conditions could replace it with a less compact tail bound.  The proof
is given in Appendix~\ref{app:proofs}.

We now combine the previous results.  Let \(\sigma^*\) be the permutation
used in Theorem~\ref{thm:partition}, and define
\[
  \widetilde C=\sigma^*\circ\widehat C,
  \qquad
  \widetilde\Phi_k=\widehat\Phi_{\sigma^{*-1}(k)}.
\]
Suppose the fitted regression networks satisfy the high-level bound
\begin{equation}\label{eq:regression_rate_high_level}
  \sum_{k=1}^{K_0}
  \mu_k E\|\widetilde\Phi_k-f_{0,k}\|_{k,2}^2
  \lesssim
  n^{-2\beta/(2\beta+p)}
\end{equation}
up to logarithmic factors.  This is the same order as the oracle benchmark
in Proposition~\ref{prop:oracle}.  Then Theorems~\ref{thm:partition} and
\ref{thm:risk} imply
\begin{equation}\label{eq:main_rate}
  E\{\widehat f(S,X)-f_0(S,X)\}^2
  \lesssim
  n^{-2\beta/(2\beta+p)}
  +
  a_n^{\gamma}.
\end{equation}
The assignment term is no larger than the oracle regression term when
\begin{equation}\label{eq:oracle_condition}
  a_n^\gamma
  =
  O\{n^{-2\beta/(2\beta+p)}\}.
\end{equation}
Under \eqref{eq:oracle_condition},
\[
  E\{\widehat f(S,X)-f_0(S,X)\}^2
  \lesssim
  n^{-2\beta/(2\beta+p)}
\]
up to logarithmic factors.  Thus consistent partition estimation alone is
not enough for the oracle conclusion: the boundary must be learned quickly
enough that assignment error does not dominate estimation of the
\(p\)-dimensional regional regressions.  Conversely, when
\eqref{eq:oracle_condition} fails, improving the regional regressions cannot remove the
leading error caused by an inaccurately located boundary.

\begin{remark}[Continuous relaxation]\label{rem:gumbel}
The results concern the final hard assignment.  The annealed softmax is an
optimization device, not a second population target.  A full analysis would
need a uniform comparison between relaxed and hard risks, with an additional
term tending to zero with the final temperature.  Such an optimization result
is beyond the scope of this paper and is not implied by
Assumption~\ref{ass:assignment_rate}.
\end{remark}

\section{Simulation studies}\label{sec:simulation}

We conducted Monte Carlo experiments to examine prediction accuracy, recovery
of the spatial partition, sensitivity to unequal cluster sizes, and the
contribution of spatial regularization.  Two experiments were used for
distinct purposes.  Experiment~1 isolates separation and region-size effects
in a two-region model with independent sampling.  Experiment~2 considers
three curved regions and adds nonlinear regression functions, spatially structured
covariates, preferential sampling, and correlated errors.  Except for the
final selection analysis, \(K\) was fixed at its true value so that the
results concern estimation rather than selection of the number of regions.

\paragraph{Experiment 1: separation and unequal region sizes.}

For each replication, the locations were generated independently from the
uniform distribution on \([0,1]^2\), and the five covariates were generated
independently from \(\operatorname{Unif}(-1,1)\), independently of location.
Let
\[
  g(x)=0.7(x_2^2-1/3)+0.45x_4,
  \qquad
  h(x)=1.6\sin(\pi x_1)+0.8x_3x_5.
\]
Conditional on the spatial label \(C_0(s)\), the response was generated from
\[
  Y=g(X)+a\{1(C_0(S)=1)-1(C_0(S)=2)\}h(X)+\epsilon,
  \qquad \epsilon\sim N(0,\sigma^2).
\]
We considered three settings.  In the balanced, strong-separation setting,
\(C_0(s)=1\) when
\(s_1\le 0.50+0.14\sin(2\pi s_2)\), and \(C_0(s)=2\) otherwise, with
\((a,\sigma)=(1,0.55)\).  The unbalanced setting used \(C_0(s)=1\) inside
the circular region
\[
 (s_1-0.33)^2+(s_2-0.58)^2\le0.30^2,
 \qquad (a,\sigma)=(1,0.55),
\]
so that the smaller cluster occupied approximately 28\% of the spatial domain.
The weak-separation setting used the balanced sinusoidal partition with
\((a,\sigma)=(0.45,0.80)\).  The sample size was \(n=300,600\), or 1200.

The proposed method used a two-layer assignment network with 12 hidden ReLU
units and two cluster-specific regression networks, each with 20 hidden ReLU
units.  We used a deterministic annealed softmax gate so that its
zero-temperature limit agrees with the hard assignment used for prediction.
The temperature was decreased geometrically from 1.5 to 0.18.  An eight-nearest
neighbor graph was used in the normalized spatial penalty
\[
 |E_n|^{-1}\sum_{(i,j)\in E_n}w_{ij}
 \|\Psi(s_i)-\Psi(s_j)\|_2^2,
\]
where \(w_{ij}=\exp\{-(\|s_i-s_j\|/b)^2\}\) and \(b\) was the median edge
length.  The spatial and balance penalty weights were 0.18 and 0.02,
respectively.  Each joint model was fit from two initializations, and the fit
with the smaller training objective was retained.

The same four implementable methods were used throughout both simulation
experiments and the real-data analysis: the proposed estimator; a global
neural network that directly regressed the response on \((S,X)\) without
imposing a partition; regularized geographically weighted regression (GWR)
with a Gaussian kernel; and a two-stage procedure that applied spatial
\(K\)-means and then fit separate neural networks.  GWR bandwidths were
selected on an independent validation subset from
\(\{0.025,0.04,0.06,0.09,0.14,0.22\}\).  In the simulations, we additionally
reported an oracle estimator that fit separate neural networks using the true
labels.  The cluster-specific network architecture was held fixed across the
clustered methods.  The oracle is not implementable and is included only as a
reference benchmark, not as a competing method.

Performance was evaluated on an independently generated test set of 2500
observations.  The primary measure was integrated squared error (ISE),
\[
  \frac{1}{2500}\sum_{i=1}^{2500}
  \{\widehat f(S_i,X_i)-f_0(S_i,X_i)\}^2.
\]
We also recorded misclassification error after optimal permutation of the
cluster labels and the adjusted Rand index (ARI).  Results are based on 100 replications; numbers
in parentheses in the tables are Monte Carlo standard errors.

Table~\ref{tab:sim-ise} summarizes prediction accuracy, and
Figure~\ref{fig:sim-ise} illustrates the fitted conditional-mean surfaces in
representative data sets.  In the balanced, strong-separation setting, the ISE of the proposed
estimator decreased from 0.886 at \(n=300\) to 0.276 at \(n=1200\), compared
with 1.203 and 0.531, respectively, for the global network.  The corresponding
oracle errors were 0.490 and 0.084.  Thus the gap between the proposed and
oracle estimators remained visible, but both improved substantially with
sample size.  At \(n=1200\), GWR had ISE 0.767 and the two-stage procedure had
ISE 1.236; the latter was substantially less accurate because
clustering the locations alone did not recover a partition that was optimal
for regression.

The same qualitative pattern was observed with unequal cluster sizes.  At
\(n=1200\), the proposed estimator had ISE 0.280, compared with 0.776 for the
global network, 0.836 for GWR, and 1.244 for the two-stage procedure.  Its test
misclassification rate decreased from 0.235 at \(n=300\) to 0.036 at
\(n=1200\), despite the balance penalty favoring equal occupancy.  This
indicates that, at the weight used here, the balance term prevented degenerate
solutions without forcing the fitted regions to have equal area.

The proposed estimator also performed best under weak separation, although
its advantage over GWR was smaller than in the strong-separation settings.
Its ISE was 0.235 at \(n=300\), compared with 0.276 for GWR, 0.829 for the
global network, and 0.980 for the two-stage procedure.  At \(n=1200\), the
corresponding values were 0.201, 0.212, 0.259, and 0.330.  Over the same range,
the proposed estimator's misclassification rate decreased from 0.158 to
0.045.  The narrower prediction gap in this setting is consistent with the
greater difficulty of identifying a partition when the regional regression
functions are less well separated.

The comparison with GWR distinguishes piecewise spatial heterogeneity from a
smoothly varying local regression surface, while the two-stage procedure
isolates the benefit of allowing the response to inform the partition.
Figure~\ref{fig:sim-partition} gives the full distribution of
partition error at \(n=1200\), and Figure~\ref{fig:sim-example} shows that the
estimated gate reproduced the curved boundary in a representative data set.

\begin{table}[t]
\centering
\caption{Experiment~1: integrated squared error by two-region scenario and
sample size. Entries are Monte Carlo means over 100 replications, with Monte
Carlo standard errors in parentheses.}
\label{tab:sim-ise}
\begin{threeparttable}
\begin{tabular}{llccc}
\toprule
Scenario & Method & \(n=300\) & \(n=600\) & \(n=1200\) \\
\midrule
Balanced, strong & Global DNN & 1.203 (0.012) & 0.775 (0.006) & 0.531 (0.005) \\
Balanced, strong & GWR & 0.905 (0.006) & 0.827 (0.005) & 0.767 (0.002) \\
Balanced, strong & \(K\)-means + DNN & 1.937 (0.083) & 1.585 (0.065) & 1.236 (0.052) \\
Balanced, strong & Proposed & 0.886 (0.034) & 0.425 (0.008) & 0.276 (0.009) \\
Balanced, strong & Oracle DNN & 0.490 (0.014) & 0.182 (0.005) & 0.084 (0.003) \\
\addlinespace
Unbalanced, strong & Global DNN & 2.021 (0.022) & 1.217 (0.014) & 0.776 (0.008) \\
Unbalanced, strong & GWR & 1.049 (0.007) & 0.911 (0.005) & 0.836 (0.003) \\
Unbalanced, strong & \(K\)-means + DNN & 2.273 (0.029) & 1.562 (0.016) & 1.244 (0.010) \\
Unbalanced, strong & Proposed & 0.842 (0.044) & 0.734 (0.034) & 0.280 (0.017) \\
Unbalanced, strong & Oracle DNN & 0.483 (0.013) & 0.196 (0.006) & 0.082 (0.002) \\
\addlinespace
Balanced, weak & Global DNN & 0.829 (0.010) & 0.473 (0.005) & 0.259 (0.002) \\
Balanced, weak & GWR & 0.276 (0.005) & 0.234 (0.003) & 0.212 (0.001) \\
Balanced, weak & \(K\)-means + DNN & 0.980 (0.023) & 0.586 (0.013) & 0.330 (0.011) \\
Balanced, weak & Proposed & 0.235 (0.023) & 0.210 (0.008) & 0.201 (0.003) \\
Balanced, weak & Oracle DNN & 0.666 (0.012) & 0.311 (0.005) & 0.137 (0.003) \\
\bottomrule
\end{tabular}
\begin{tablenotes}[flushleft]\footnotesize
\item The oracle estimator uses the true spatial labels and is not an
implementable competitor.
\end{tablenotes}
\end{threeparttable}
\end{table}

\begin{table}[t]
\centering
\caption{Experiment~1: recovery of the two-region spatial partition at
\(n=1200\). Entries are Monte Carlo means over 100 replications, with Monte
Carlo standard errors in parentheses.}
\label{tab:sim-partition}
\begin{tabular}{llcc}
\toprule
Scenario & Method & Misclassification & ARI \\
\midrule
Balanced, strong & \(K\)-means + DNN & 0.285 (0.016) & 0.290 (0.032) \\
Balanced, strong & Proposed & 0.034 (0.001) & 0.870 (0.005) \\
\addlinespace
Unbalanced, strong & \(K\)-means + DNN & 0.351 (0.005) & 0.098 (0.006) \\
Unbalanced, strong & Proposed & 0.036 (0.004) & 0.866 (0.013) \\
\addlinespace
Balanced, weak & \(K\)-means + DNN & 0.219 (0.016) & 0.415 (0.031) \\
Balanced, weak & Proposed & 0.045 (0.001) & 0.827 (0.004) \\
\bottomrule
\end{tabular}
\end{table}

\begin{figure}[t]
\centering
\includegraphics[width=\textwidth]{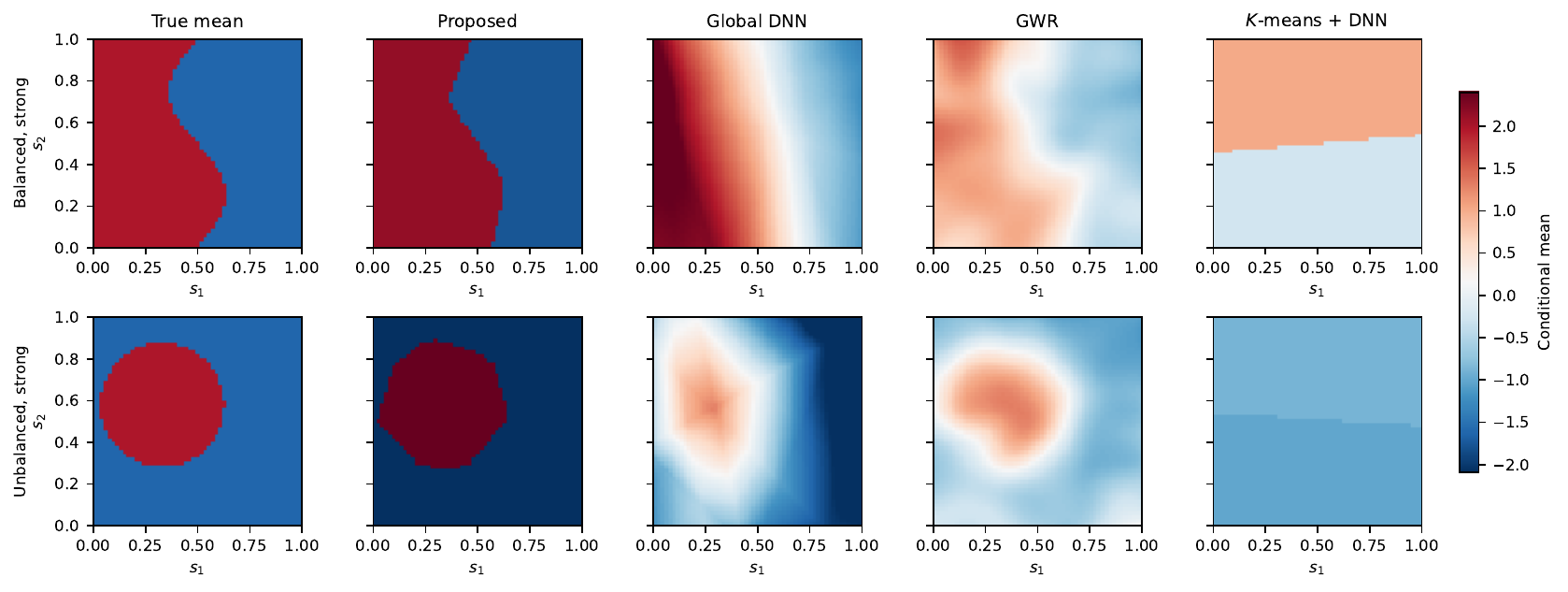}
\caption{Experiment~1: true and estimated conditional-mean surfaces in
prespecified representative data sets with \(n=1200\).  The first row is the
balanced strong-separation setting and the second is the unbalanced
strong-separation setting.  From left to right, the panels show the true mean,
the proposed estimator, the global DNN, GWR, and spatial \(K\)-means followed
by separate DNNs.  The nonspatial
covariates are fixed at
\(x^\star=(0.5,0.5,0.5,0.5,0.5)^\top\), and all panels use the same color
scale.  Data sets were generated using fixed seeds before model fitting and
were not selected from the Monte Carlo replications.}
\label{fig:sim-ise}
\end{figure}

\begin{figure}[t]
\centering
\includegraphics[width=0.72\textwidth]{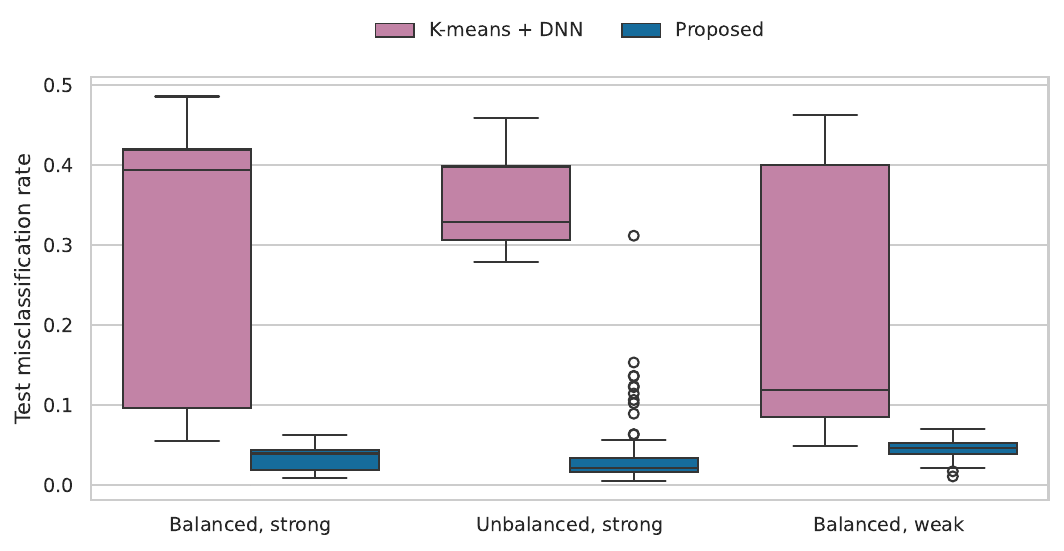}
\caption{Experiment~1: distribution over 100 replications of
label-permutation-adjusted test misclassification rates at \(n=1200\), shown
separately for the balanced strong, unbalanced strong, and balanced weak
settings.}
\label{fig:sim-partition}
\end{figure}

\begin{figure}[t]
\centering
\includegraphics[width=0.78\textwidth]{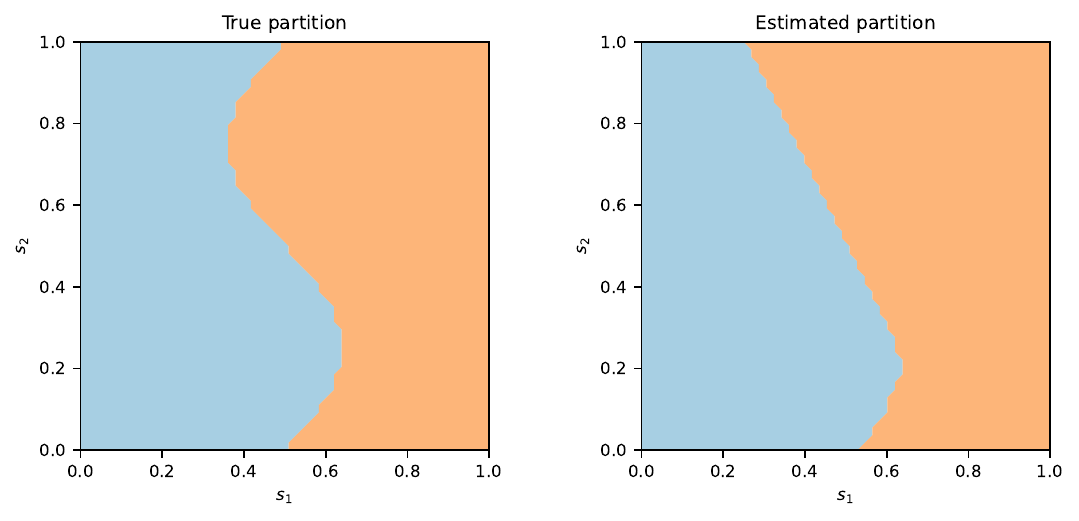}
\caption{Experiment~1: true and proposed-method partitions in a representative
balanced strong-separation data set with \(n=800\).  Cluster labels have been
permuted to match the true labels.}
\label{fig:sim-example}
\end{figure}

\FloatBarrier
\paragraph{Experiment 2: spatial dependence and multiple regions.}
\label{sec:sim-spatial-robustness}

The preceding experiments used two regions, independent covariates, and
independent errors.  We conducted a second set of experiments to examine three
features that arise routinely in spatial applications: more than two regions,
spatially correlated covariates, and nonuniform sampling.  The same
implementable comparison methods were retained to make performance directly
comparable across experiments.

Let $S=(S_1,S_2)\in[0,1]^2$.  The true partition contained three regions,
\[
 C_0(s)=
 \begin{cases}
  1, & s_1\le b_1(s_2),\\
  2, & b_1(s_2)<s_1\le b_2(s_2),\\
  3, & s_1>b_2(s_2),
 \end{cases}
\]
where
\[
 b_1(v)=0.30+0.065\sin(2\pi v),\qquad
 b_2(v)=0.68+0.075\sin(2\pi v+\pi/3).
\]
For $j=1,\ldots,5$, the covariates were generated as
\[
 X_j(s)=\sqrt{0.55}\,Z_j(s)+\sqrt{0.45}\,\eta_j,
 \qquad \eta_j\sim N(0,1),
\]
where the $Z_j$ were independent, mean-zero Gaussian processes with covariance
$\operatorname{cov}\{Z_j(s),Z_j(t)\}=\exp\{-\|s-t\|^2/(2\times0.22^2)\}$.
Thus nearby observations had similar, but not identical, covariate
distributions.  The regression mean was
\[
 f_0(s,x)=\alpha_{C_0(s)}+x^\top\beta_{C_0(s)}+r_{C_0(s)}(x),
\]
with
\[
 \alpha=(0.90,-0.50,0.25)^\top,\qquad
 \begin{pmatrix}\beta_1^\top\\ \beta_2^\top\\ \beta_3^\top\end{pmatrix}
 =
 \begin{pmatrix}
  1.15&-0.65& 0.35& 0.00& 0.55\\
 -0.90& 0.80& 0.65&-0.45& 0.10\\
  0.35& 0.20&-1.05& 0.85&-0.55
 \end{pmatrix}.
\]
The linear setting used $r_k\equiv0$.  In the nonlinear settings,
\begin{align*}
 r_1(x)&=0.95\sin(\pi x_1/2)+0.45x_2x_3,\\
 r_2(x)&=0.65(x_1^2-1)+0.80\sin(\pi x_4/2),\\
 r_3(x)&=0.75x_2x_5-0.70\cos(\pi x_3/2).
\end{align*}

We considered three scenarios.  In L-IID, the regression mean was linear,
locations were uniform on $[0,1]^2$, and the errors were independent
$N(0,0.6^2)$.  NL-IID replaced the linear regressions by the nonlinear functions
above.  In NL-SP, the training locations followed the mixture
\[
 0.55\,\operatorname{Unif}([0,1]^2)
 +0.45\{\operatorname{Beta}(2,5)\otimes\operatorname{Beta}(5,2)\},
\]
whereas test locations remained uniform.  Its errors were
\[
 \epsilon(s)=0.6\{\sqrt{0.55}\,U(s)+\sqrt{0.45}\,e(s)\},
\]
where $e(s)$ were independent standard normal variables and $U$ was a
mean-zero Gaussian process with squared-exponential range 0.16.  Independent
random Fourier-feature representations were used to generate the Gaussian processes.
The sample size was $n=600$ or 1200, and each design was replicated 100 times.

The proposed estimator used $K=3$, 32 hidden units in the assignment network,
and 20 hidden units in each cluster-specific regression network.  The spatial and balance penalty weights
were 0.08 and 0.002.  The deterministic temperature was annealed from 1.5 to
0.14 over 650 epochs.  In addition to a spatial $K$-means initialization, we
used a response-informed initialization obtained by fitting local ridge
regressions over $\max\{40,\lfloor n^{0.60}\rfloor\}$ nearest neighbors and
clustering the resulting coefficient estimates.

The comparison methods were the same as in Experiment~1: a global DNN fitted
to $(S,X)$, regularized GWR with a Gaussian kernel, and spatial $K$-means
followed by separate DNNs.  The oracle DNN fitted using the true labels was
again included only as a simulation reference.  GWR bandwidths were selected
on an independent validation subset from
$\{0.025,0.04,0.06,0.09,0.14,0.22\}$.  The same cluster-specific network
architecture was used by the proposed, two-stage, and oracle estimators.

Performance was evaluated at 2500 independent, uniformly distributed test
locations.  The primary measure was ISE for the structural regression mean.
We also computed permutation-adjusted misclassification error and ARI.
Numbers in parentheses below are Monte Carlo standard errors.

Table~\ref{tab:sim2-ise} summarizes prediction accuracy, and
Figure~\ref{fig:sim2-ise} illustrates the fitted conditional-mean surfaces in
representative nonlinear settings.  The proposed estimator improved
substantially as the sample size increased in all three settings.  It also had
the smallest ISE among the four implementable methods in every Experiment~2
configuration.  Table~\ref{tab:sim2-partition} compares its response-informed
partition with the partition obtained from locations alone by the two-stage
procedure, and Figure~\ref{fig:sim2-partition} illustrates the fitted
partitions in a representative NL-SP data set.

The nonlinear settings more clearly separated the methods.  In NL-IID with
$n=1200$, the proposed ISE was 0.344, compared with 0.731 for GWR, 1.028 for
the global DNN, and 2.016 for the two-stage method.  Under preferential
sampling and spatially correlated errors, the corresponding values were
0.558, 0.897, 1.147, and 2.272.  Thus spatial
dependence and nonuniform sampling increased the gap to
the oracle, but did not remove the benefit of learning a piecewise nonlinear
regression surface.

Together, the results support learning the spatial assignment jointly with
nonlinear regional regressions when the response surface changes sharply
across curved boundaries.  The comparison with the two-stage method shows
that spatial clustering alone was insufficient in these designs, while the
comparison with GWR shows the benefit of accommodating discontinuities rather
than imposing a smoothly varying local regression surface.

\begin{table}[t]
\centering
\caption{Experiment~2: integrated squared error in the three-region
simulation. Entries are Monte Carlo means over 100 replications, with Monte
Carlo standard errors in parentheses.}
\label{tab:sim2-ise}
\begin{threeparttable}
\centering
\begin{adjustbox}{max width=\textwidth, center}
\begin{tabular}{llcc}
\toprule
Scenario & Method & $n=600$ & $n=1200$\\
\midrule
L-IID & Global DNN & 1.205 (0.033) & 0.737 (0.019)\\
      & GWR & 0.525 (0.018) & 0.382 (0.012)\\
      & \(K\)-means + DNN & 2.084 (0.069) & 1.560 (0.050)\\
      & Proposed & 0.480 (0.017) & 0.263 (0.010)\\
      & Oracle DNN & 0.138 (0.003) & 0.060 (0.001)\\
\addlinespace
NL-IID & Global DNN & 1.469 (0.028) & 1.028 (0.024)\\
       & GWR & 0.883 (0.014) & 0.731 (0.012)\\
       & \(K\)-means + DNN & 2.504 (0.065) & 2.016 (0.051)\\
       & Proposed & 0.633 (0.020) & 0.344 (0.010)\\
       & Oracle DNN & 0.209 (0.003) & 0.107 (0.002)\\
\addlinespace
NL-SP & Global DNN & 1.708 (0.037) & 1.147 (0.027)\\
      & GWR & 1.088 (0.019) & 0.897 (0.014)\\
      & \(K\)-means + DNN & 3.063 (0.077) & 2.272 (0.061)\\
      & Proposed & 0.841 (0.032) & 0.558 (0.027)\\
      & Oracle DNN & 0.290 (0.008) & 0.198 (0.006)\\
\bottomrule
\end{tabular}
\end{adjustbox}
\begin{tablenotes}
\footnotesize
\item L-IID: linear regressions and independent errors; NL-IID: nonlinear regressions
and independent errors; NL-SP: nonlinear regressions, preferential sampling, and
spatially correlated errors.  The oracle uses the true labels.
\end{tablenotes}
\end{threeparttable}
\end{table}

\begin{table}[t]
\centering
\caption{Experiment~2: recovery of the three-region partition at \(n=1200\).
Entries are Monte Carlo means over 100 replications, with Monte Carlo standard
errors in parentheses.}
\label{tab:sim2-partition}
\begin{tabular}{llcc}
\toprule
Scenario & Method & Misclassification & ARI\\
\midrule
L-IID & \(K\)-means + DNN & 0.385 (0.003) & 0.276 (0.007)\\
      & Proposed & 0.027 (0.002) & 0.922 (0.004)\\
\addlinespace
NL-IID & \(K\)-means + DNN & 0.385 (0.003) & 0.277 (0.007)\\
       & Proposed & 0.026 (0.001) & 0.923 (0.004)\\
\addlinespace
NL-SP & \(K\)-means + DNN & 0.384 (0.005) & 0.292 (0.009)\\
      & Proposed & 0.049 (0.006) & 0.870 (0.009)\\
\bottomrule
\end{tabular}
\end{table}

\begin{figure}[t]
\centering
\includegraphics[width=\textwidth]{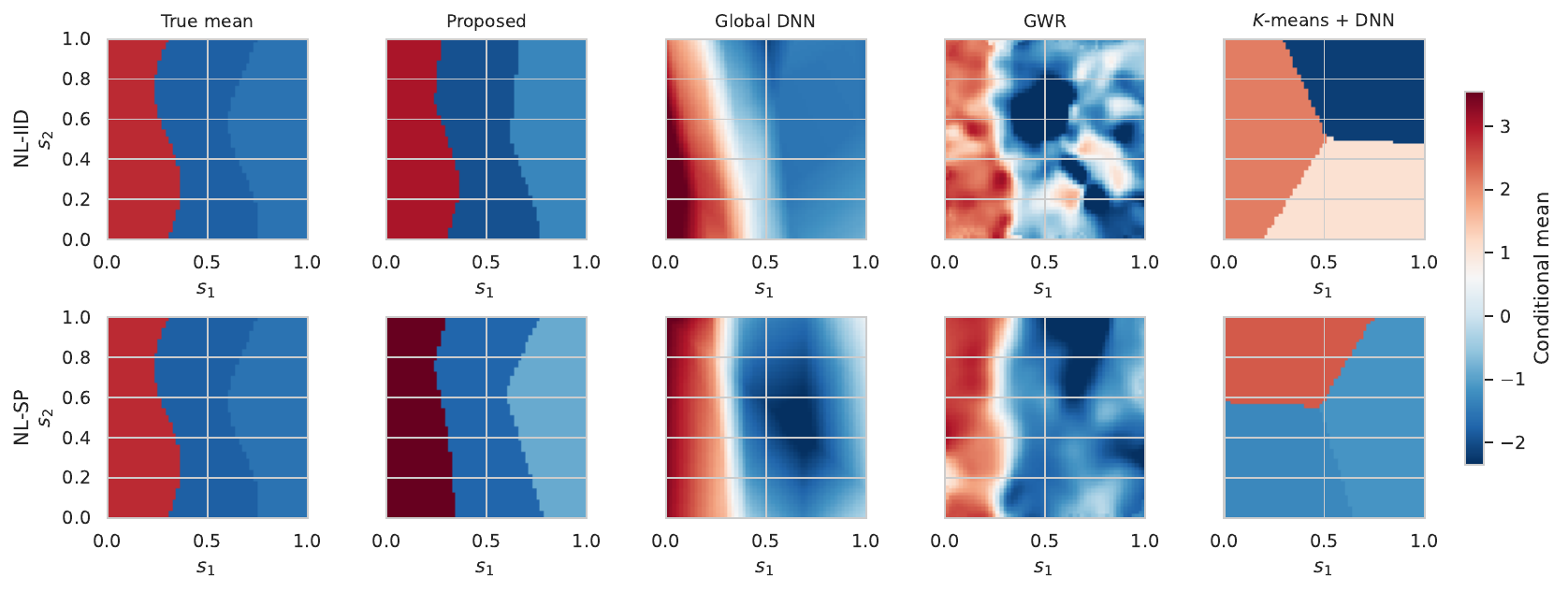}
\caption{Experiment~2: true and estimated conditional-mean surfaces in
prespecified representative data sets with \(n=1200\).  The first row is
NL-IID; the second is NL-SP, with preferential training locations and
spatially correlated errors.  From left to right, the panels show the true
mean, the proposed estimator, the global DNN, GWR, and spatial $K$-means
followed by separate DNNs.  The nonspatial covariates are fixed at
\(x^\star=(0.5,-0.5,0.5,-0.5,0.5)^\top\), and all panels use the same color
scale.  Fixed seeds were specified before fitting; the displayed data sets
were not selected from the Monte Carlo replications.}
\label{fig:sim2-ise}
\end{figure}

\begin{figure}[t]
\centering
\includegraphics[width=0.94\textwidth]{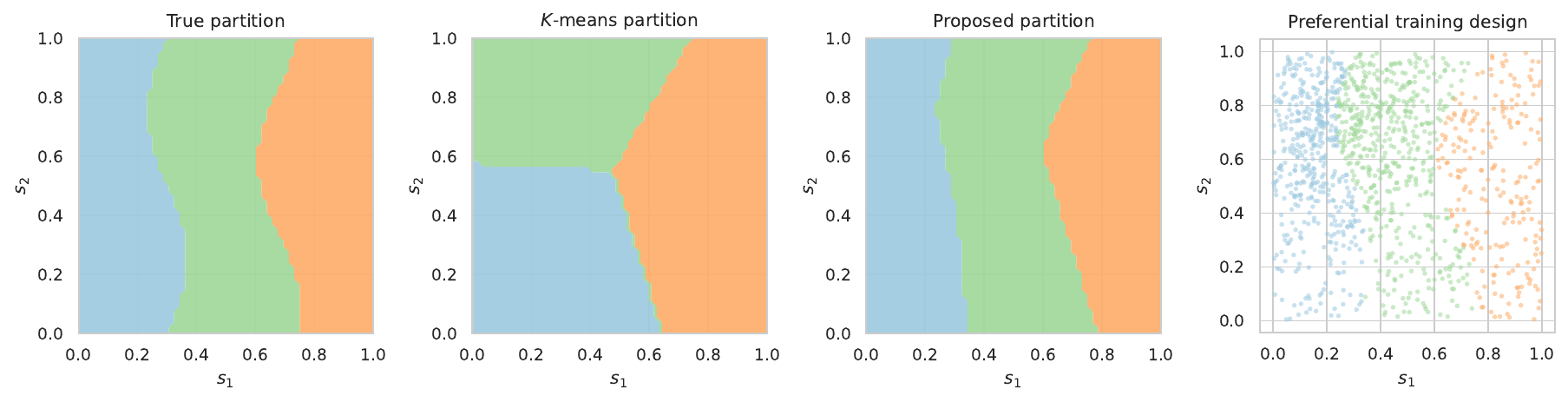}
\caption{Experiment~2: true and estimated partitions in the prespecified
representative NL-SP data set, together with the preferential training
locations.  The panels identify the displayed method, cluster labels are
permutation adjusted, and the data set was generated with a fixed seed rather
than selected from the Monte Carlo replications.}
\label{fig:sim2-partition}
\end{figure}

\FloatBarrier
\paragraph{Selection of the number of regions.}\label{sec:k-selection}
We separately assessed the cross-validation recommendation in
Section~\ref{sec:estimation}.  For each nonlinear scenario, $K$ was selected from
$\{2,3,4,5\}$ by three-fold cross-validation.  We used the one-standard-error
rule, choosing the smallest $K$ whose mean validation loss was within one
standard error of the minimum.  Candidate models used one response-informed
initialization and 360 epochs; after selection, the model was refitted to the
full training sample using the two initializations and 650 epochs described
above.  This experiment also used 100 replications.

Table~\ref{tab:sim2-k} and Figure~\ref{fig:sim2-k} show that prediction-based
cross-validation did not consistently recover the structural value $K_0=3$.
In NL-IID, the correct selection proportion decreased from 0.71 at $n=600$ to
0.52 at $n=1200$.  In NL-SP, it decreased from 0.51 to 0.29, with $K=4$ or 5
selected in 71\% of the larger-sample replications.  This behavior is not
surprising: a model with $K>K_0$ can split a true region while retaining nearly
the same regression surface, and the one-standard-error tolerance contracts
as $n$ increases.  Prediction cross-validation therefore selects predictive
complexity, not necessarily the minimal structural partition.  Unless a
separate complexity or stability criterion is imposed, $K$ should be treated
as prespecified when structural interpretation is the objective.

\begin{table}[t]
\centering
\caption{Experiment~2: selection of \(K\) by three-fold cross-validation with
the one-standard-error rule.  Selection columns give percentages over 100
replications in the NL-IID and NL-SP settings.}
\label{tab:sim2-k}
\begin{tabular}{llrrrrc}
\toprule
Scenario & $n$ & $K=2$ & $K=3$ & $K=4$ & $K=5$ & Selected-model ISE\\
\midrule
NL-IID & 600  & 2 & 71 & 21 & 6  & 0.680 (0.022)\\
       & 1200 & 0 & 52 & 35 & 13 & 0.369 (0.011)\\
NL-SP  & 600  & 5 & 51 & 30 & 14 & 0.863 (0.029)\\
       & 1200 & 0 & 29 & 44 & 27 & 0.546 (0.019)\\
\bottomrule
\end{tabular}
\end{table}

\begin{figure}[H]
\centering
\includegraphics[width=0.88\textwidth]{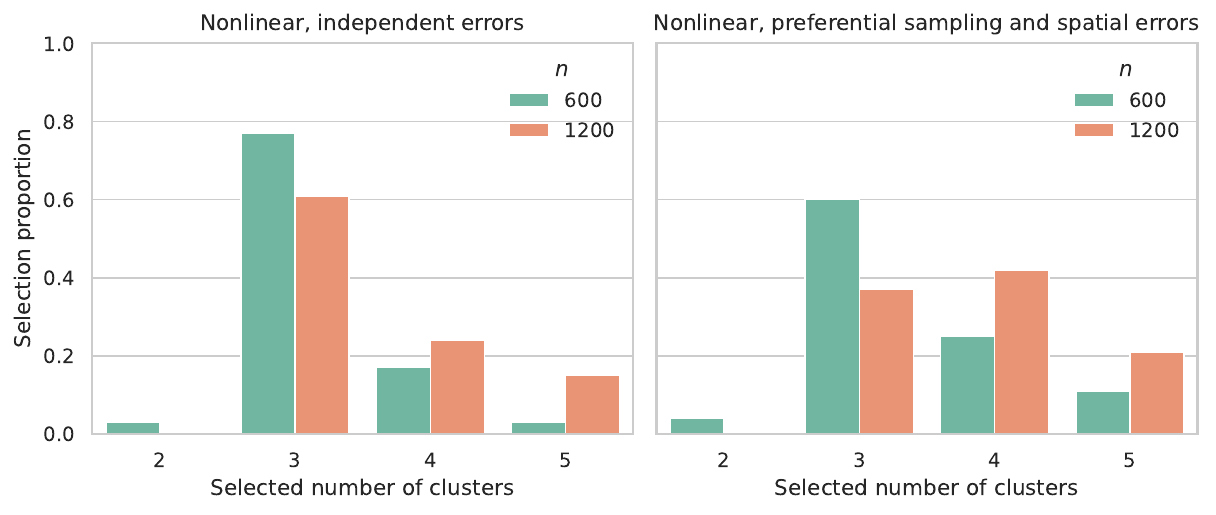}
\caption{Experiment~2: empirical distribution of the number of regions
selected by three-fold cross-validation with the one-standard-error rule in
the NL-IID and NL-SP settings.  The vertical reference indicates the true
value \(K_0=3\).}
\label{fig:sim2-k}
\end{figure}
\FloatBarrier
\section{King County house sales}\label{sec:king-county}

We applied the proposed method to residential sales in King County, Washington.
The data contain 21,613 transactions recorded between May 2014 and May 2015,
together with property characteristics and geographic coordinates
\citep{geoda2020kingcounty}.  The response was log sale price.  The covariate
vector included bedrooms, bathrooms, floors, waterfront status, view,
condition, construction grade, living area, lot area, above-ground and
basement area, age at sale, renovation status, sale date, and the living and
lot areas of nearby properties.  Area variables were transformed by
$\log(1+x)$.  Within each training sample, continuous covariates were
winsorized at the 0.5th and 99.5th percentiles and standardized.  Longitude
and latitude were supplied only to the assignment network, except in the
global neural-network comparator, which received both location and the
property covariates.

Our primary purpose in this analysis was to assess whether the method could
produce a stable and interpretable regional summary, while retaining a fair
comparison with methods designed primarily for prediction.  We prespecified
three regions to obtain a low-dimensional description of the county's major
housing submarkets.  This choice also avoids
interpreting prediction-based selection as consistent estimation of a unique
structural partition; as observed in Section~\ref{sec:k-selection},
cross-validation can favor
additional clusters that split a region without materially changing its
regression surface.  A sensitivity analysis that selected $K$ from
$\{2,3,4,5\}$ within each outer training sample chose $(5,5,2,5,5)$ across the
five folds, but the resulting five-region full-data partition was appreciably
less stable under subsampling than the three-region partition.

The assignment network had one hidden layer with 24 ReLU units, and each
cluster-specific regression network had one hidden layer with 28 ReLU units.  The deterministic softmax
temperature was annealed from 1.25 to 0.16.  We used an eight-nearest-neighbor
graph, with spatial and balance penalty weights 0.005 and 0.002, respectively.
For each training sample, spatial $K$-means and response-informed local-ridge
initializations were compared on an internal spatial validation set.  After
selecting the initialization, the model was refitted to the complete outer
training sample.

Prediction was assessed by five-fold spatial block cross-validation.  We first
formed 100 compact blocks by clustering the coordinates and then allocated
whole blocks to five folds with approximately equal sample sizes.  Thus,
nearby sales within the same block could not be divided between training and
test samples.  The comparison methods were the same three implementable
comparators used in both simulation experiments: a global DNN fitted to
$(S,X)$, regularized GWR, and spatial $K$-means followed by separate DNNs.
GWR neighborhood size was chosen from
$\{128,256,512\}$ within each outer training sample.  Table~\ref{tab:king-pred}
reports means across the five held-out folds, with standard errors computed
from the fold-specific values.

\begin{table}[H]
\centering
\caption{Spatially blocked prediction of King County log sale price. Entries
are means across five outer folds, with standard errors in parentheses. Dollar
MAE is reported in thousands of dollars.}
\label{tab:king-pred}
\resizebox{\textwidth}{!}{%
\begin{tabular}{lccrrrr}
\toprule
Method & Log RMSE & Log MAE & $R^2$ & Dollar MAE & MAPE & Moran's $I$ \\
\midrule
Global DNN
  & 0.189 (0.003) & 0.139 (0.002) & 0.864 (0.014)
  & 78.1 (4.3) & 0.142 (0.002) & 0.233 (0.023) \\
GWR
  & 0.196 (0.005) & 0.144 (0.003) & 0.854 (0.013)
  & 81.9 (5.0) & 0.148 (0.003) & 0.270 (0.034) \\
$K$-means + DNN
  & 0.237 (0.009) & 0.180 (0.006) & 0.786 (0.022)
  & 99.9 (7.2) & 0.184 (0.008) & 0.384 (0.043) \\
Proposed
  & 0.248 (0.014) & 0.186 (0.011) & 0.756 (0.045)
  & 104.5 (6.3) & 0.190 (0.010) & 0.460 (0.053) \\
\bottomrule
\end{tabular}
}
\end{table}

The global DNN had the smallest prediction error, followed by GWR.  The
proposed estimator had mean log-scale RMSE 0.248, compared with 0.189 for the
global DNN and 0.196 for GWR.  It also did not improve on the two-stage
$K$-means estimator in this application.  These results do not support a predictive advantage for a hard
piecewise regression representation on this data set.  The residual Moran's
$I$ of 0.460, compared with 0.233 for the global DNN, indicates that a smooth
within-region location effect remains after the fitted partition is taken into
account.

The fitted partition was nevertheless geographically coherent and highly
reproducible (Figure~\ref{fig:king-maps}).  The eastern/northeastern region
contained 5,644 sales and had median price \$575,000; the western/northern
region contained 9,275 sales and had median price \$520,000; and the southern
region contained 6,694 sales and had median price \$293,000.  The average
maximum assignment probabilities were 0.920, 0.935, and 0.922, respectively.
Across five fits to independent 85\% subsamples, the adjusted Rand index
relative to the full-data partition averaged 0.962 and ranged from 0.929 to
0.986.  Hence the three broad regions were not an artifact of a single
initialization or a small subset of transactions.

\begin{figure}[t]
\centering
\includegraphics[width=0.96\textwidth]{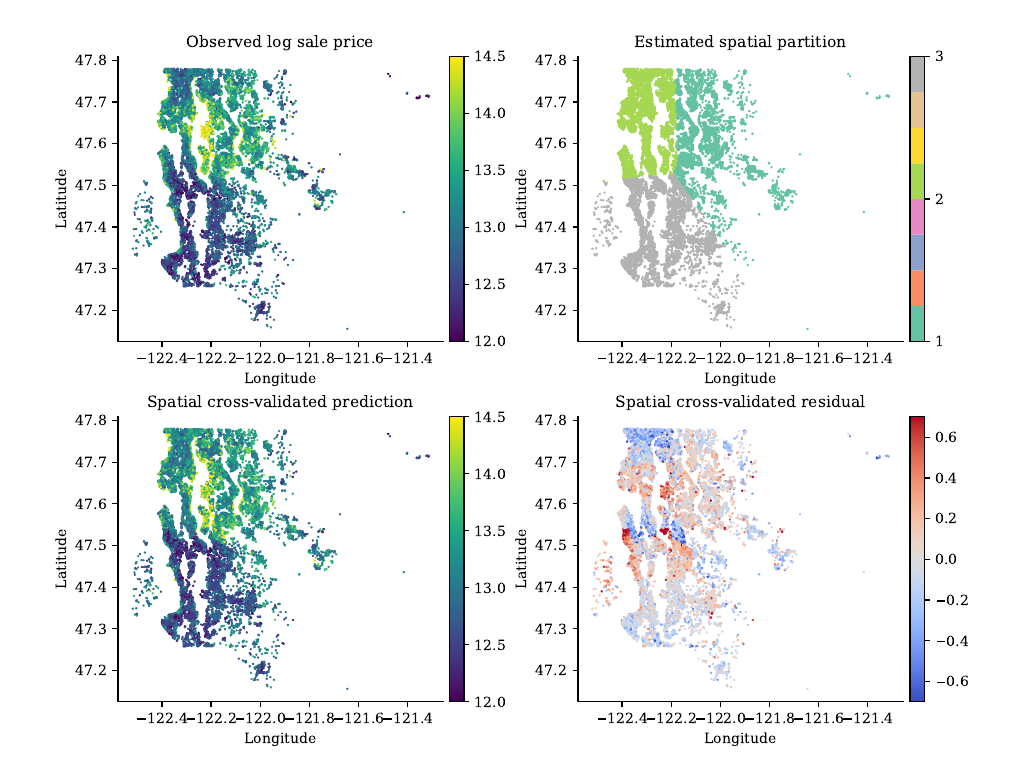}
\caption{Observed log sale prices, the full-data three-region partition,
out-of-fold predictions from the proposed method, and out-of-fold residuals.
Cluster labels are arbitrary.  Whole spatial blocks, rather than individual
transactions, were held out during cross-validation.}
\label{fig:king-maps}
\end{figure}

The cluster-specific regression networks also gave different descriptions of property
value.  Figure~\ref{fig:king-curves} varies living area while holding the other
covariates at their within-region medians.  At comparable living area, fitted
prices were substantially lower in the southern region.  Average local
sensitivities provided a second summary of heterogeneity
(Figure~\ref{fig:king-sensitivity}).  Construction grade had its largest
effect in the western/northern region, whereas waterfront status and view had
their largest effects in the southern region.  Lot area had a positive average
effect in the eastern and southern regions but a slightly negative effect in
the western/northern region after the other property characteristics were held
fixed.

\begin{figure}[t]
\centering
\includegraphics[width=0.70\textwidth]{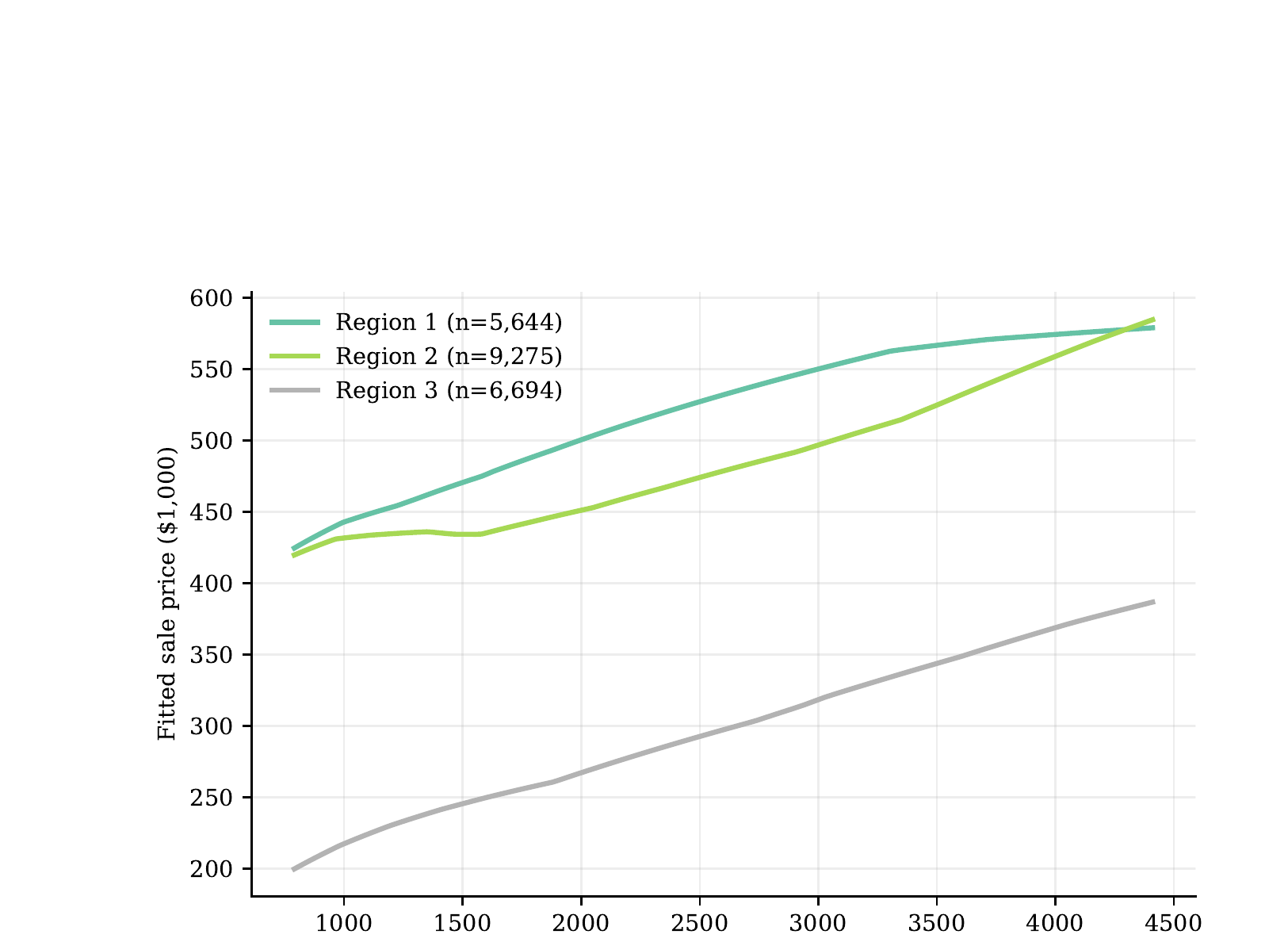}
\caption{Cluster-specific fitted sale price as a function of living area.  All
other property covariates are fixed at their within-region medians.}
\label{fig:king-curves}
\end{figure}

\begin{figure}[t]
\centering
\includegraphics[width=0.70\textwidth]{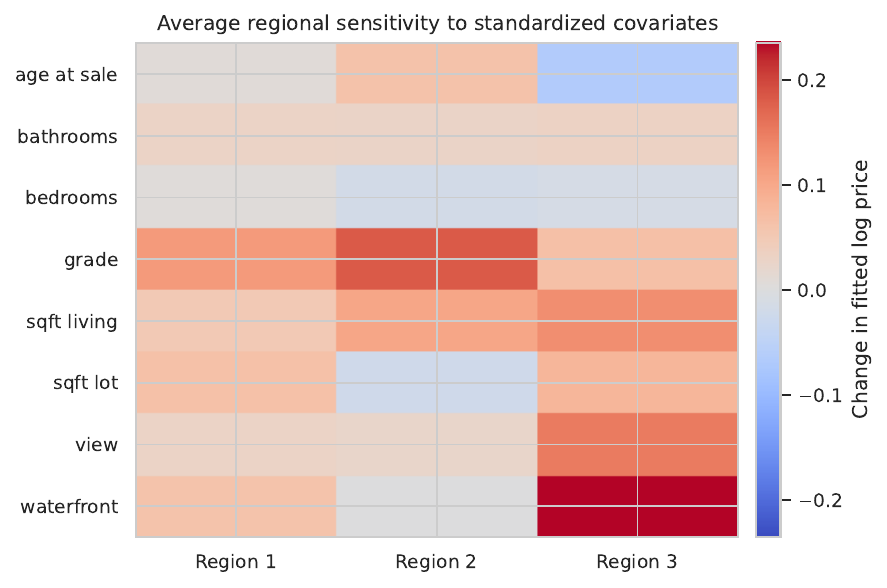}
\caption{Average within-region derivatives of the fitted regional regressions with respect
to standardized property covariates.  Positive values indicate an increase in
fitted log price.}
\label{fig:king-sensitivity}
\end{figure}

This application illustrates both the use and the limitation of the proposed
model.  It produced a stable, readily summarized partition and heterogeneous
within-region response functions, but the continuous spatial variation in
house prices was better predicted by methods that allowed location to enter
the regression surface directly.  In applications where prediction is the
sole objective, the spatially clustered model should therefore be compared
with a sufficiently flexible global spatial learner.  Its principal advantage
in the present analysis is the explicit and reproducible submarket
representation rather than lower prediction error.
\clearpage
\section{Discussion}\label{sec:discussion}

The proposed model is intended for settings in which regression
relationships are approximately homogeneous within regions and may change
abruptly at their boundaries.  It is not a general replacement for smooth
spatial regression.  This distinction was visible in the numerical work.  The
method performed well in the piecewise nonlinear simulations, whereas the
King County data favored a global spatial neural network for prediction.  The
three-region housing fit was useful for a different reason: it gave a stable
and readily summarized description of geographic submarkets.

The theoretical results also reflect this intended setting.  Separation and
covariate overlap identify the partition and cluster-specific functions up to
label permutation, and the prediction bound isolates the cost of assigning a
new location to the wrong region.  The oracle-rate result shows that estimating
the partition need not alter the leading regression rate when assignment error
decays sufficiently quickly.  In the revised formulation, this requires
\(a_n^\gamma=O\{n^{-2\beta/(2\beta+p)}\}\), where \(a_n\) is the
boundary-localization rate.  This separates a geometric requirement on the
estimated partition from the nonparametric rate of the regional regressions.
When the condition fails, assignment error may dominate even if the
regression networks are accurately estimated.

Several extensions are natural. First, the current theory takes
\(K=K_0\), while in practice \(K\) is selected by cross-validation.
Developing a consistent data-adaptive procedure for \(K\) would be
useful. Second, the model assumes homogeneous noise variance; allowing
cluster-specific variances is straightforward in the model but would
require modifications to the theory. Third, generalized outcomes such as
binary or count responses can be handled by replacing the squared-error
loss with an appropriate negative log-likelihood.
Fourth, Assumption~\ref{ass:support} requires the conditional covariate
distribution to have a density uniformly bounded away from zero and
infinity across all spatial locations. This rules out settings where
different spatial regions have substantially different covariate
distributions. Relaxing this to a weaker overlap condition, or allowing
the support of \(\nu\) to vary with the cluster, would be a useful
direction for future work.

\clearpage
\appendix
\section{Proofs of the main results}\label{app:proofs}

\begin{proof}[Proof of Theorem~\ref{thm:identifiability}]
Using the model and \(E(\epsilon\mid S,X)=0\),
\[
\begin{aligned}
  R(q,\phi)
  &=
  E\left\{f_{0,C_0(S)}(X)+\epsilon-\phi_{q(S)}(X)\right\}^2  \\
  &=
  E\left\{f_{0,C_0(S)}(X)-\phi_{q(S)}(X)\right\}^2
  +E(\epsilon^2).
\end{aligned}
\]
Therefore
\begin{equation}\label{eq:excess_risk}
  R(q,\phi)-R_0
  =
  E\left\{f_{0,C_0(S)}(X)-\phi_{q(S)}(X)\right\}^2
  \geq 0 .
\end{equation}
The oracle pair \((C_0,f_0)\) attains zero excess risk.  Hence any
population minimizer must satisfy
\begin{equation}\label{eq:zero_excess}
  f_{0,C_0(S)}(X)=\phi_{q(S)}(X)
\end{equation}
almost surely.

Let
\[
  A_a=\{s:q(s)=a\},
  \qquad a=1,\ldots,K_0.
\]
Fix a fitted label \(a\).  Suppose \(A_a\) has positive-probability overlap
with a true cluster \(\mathcal C_k\).  On the event
\(\{S\in A_a\cap\mathcal C_k\}\), equation \eqref{eq:zero_excess} gives
\[
  \phi_a(X)=f_{0,k}(X)
\]
almost surely.  Equivalently,
\[
  0
  =
  E\left[
    \{\phi_a(X)-f_{0,k}(X)\}^2
    1\{S\in A_a\cap\mathcal C_k\}
  \right].
\]
Writing this expectation by conditioning on \(S=s\),
\[
  0
  =
  \int_{A_a\cap\mathcal C_k}
  \int_{\mathcal X}
  \{\phi_a(x)-f_{0,k}(x)\}^2
  p_s(x)\,d\nu(x)\,dP_S(s).
\]
Because \(p_s(x)\geq m_x>0\) for \(\nu\)-almost every \(x\) and almost every
\(s\), and because \(P(S\in A_a\cap\mathcal C_k)>0\), it follows that
\[
  \phi_a=f_{0,k}
\]
\(\nu\)-almost everywhere.

Now suppose the same fitted label \(a\) also overlaps another true cluster
\(\mathcal C_l\), with \(l\neq k\), in positive probability.  The same
argument gives
\[
  \phi_a=f_{0,l}
\]
\(\nu\)-almost everywhere.  Thus \(f_{0,k}=f_{0,l}\) in \(L^2(\nu)\), which
contradicts Assumption~\ref{ass:separation}.  Therefore each fitted label
can overlap, up to null sets, with at most one true cluster.

It remains to show that every true cluster is represented by some fitted
label.  Since the sets \(A_1,\ldots,A_{K_0}\) form a partition of
\(\mathcal S\),
\[
  \mathcal C_k
  =
  \bigcup_{a=1}^{K_0}(A_a\cap\mathcal C_k).
\]
By Assumption~\ref{ass:size}, \(P(S\in\mathcal C_k)=\mu_k>0\).  Hence at
least one set \(A_a\cap\mathcal C_k\) must have positive probability.

There are \(K_0\) fitted labels and \(K_0\) true clusters.  Each fitted
label overlaps at most one true cluster, and every true cluster is overlapped
by at least one fitted label.  Thus the correspondence is one-to-one.  Hence
there exists a permutation \(\sigma\) such that
\[
  q(s)=\sigma\{C_0(s)\}
\]
almost surely.  On each matched cluster, equation \eqref{eq:zero_excess}
also gives
\[
  \phi_{\sigma(k)}=f_{0,k}
\]
\(\nu\)-almost everywhere.  This proves the theorem.
\end{proof}

\begin{proof}[Proof of Theorem~\ref{thm:partition}]
Let
\[
  D_n=\delta_\partial(\widehat C,C_0).
\]
By the definition in \eqref{eq:boundary_localization}, for every
\(\eta>0\) there is a permutation \(\sigma_\eta\) such that
\[
 \{s:\sigma_\eta(\widehat C(s))\neq C_0(s)\}
 \subseteq \mathcal B_0^{D_n+\eta}.
\]
Conditional on the training sample, \(D_n\) is fixed and the test location
\(S\) is independent of that sample.  Assumption~\ref{ass:margin} therefore
gives, whenever \(D_n+\eta<t_0\),
\[
\begin{aligned}
 \inf_{\sigma\in\mathfrak S_{K_0}}
 P_S\{\sigma(\widehat C(S))\neq C_0(S)\mid\mathcal D_n\}
 &\leq P_S(S\in\mathcal B_0^{D_n+\eta})\\
 &\leq c_0(D_n+\eta)^\gamma .
\end{aligned}
\]
The left side is at most one, so the same conclusion up to a fixed constant
holds when \(D_n+\eta\geq t_0\).  Letting \(\eta\downarrow0\), taking
expectations, and applying Assumption~\ref{ass:assignment_rate} yields
\[
 E\left[
  \inf_{\sigma\in\mathfrak S_{K_0}}
  P_S\{\sigma(\widehat C(S))\neq C_0(S)\mid\mathcal D_n\}
 \right]
 \lesssim E(D_n^\gamma)
 \lesssim a_n^\gamma ,
\]
which proves the result.
\end{proof}

\begin{proof}[Proof of Proposition~\ref{prop:oracle}]
Let
\[
  n_k=\#\{i:S_i\in\mathcal C_k\}.
\]
By Assumption~\ref{ass:size},
\[
  E(n_k)=n\mu_k\geq n\mu_0.
\]
Standard concentration gives
\[
  P(n_k<n\mu_0/2)\leq \exp(-cn)
\]
for some constant \(c>0\).

The conditional distribution of \(X\) given \(S\in\mathcal C_k\) has density
\[
  p_k(x)
  =
  \frac{1}{\mu_k}
  \int_{\mathcal C_k}p_s(x)\,dP_S(s)
\]
with respect to \(\nu\).  By Assumption~\ref{ass:support},
\[
  m_x\leq p_k(x)\leq M_x
\]
for \(\nu\)-almost every \(x\).  Hence the norm \(\|\cdot\|_{k,2}\) is
equivalent to the \(L^2(\nu)\) norm on this cluster.

Work on the event \(n_k\geq n\mu_0/2\).  Let \(\mathcal F_n\) be a bounded
ReLU network class with \(N_n\) effective parameters and depth of order
\(\log n\).  Standard approximation results for H\"older functions on a
compact subset of \(\mathbb R^p\)
\citep{yarotsky2017error,schmidthieber2020nonparametric} imply that there exists
\(\widetilde f_k\in\mathcal F_n\) such that
\[
  \|\widetilde f_k-f_{0,k}\|_\infty
  \leq
  C N_n^{-\beta/p}.
\]
Therefore
\[
  \|\widetilde f_k-f_{0,k}\|_{k,2}^2
  \leq
  C N_n^{-2\beta/p}.
\]
A standard least-squares oracle inequality for bounded finite-entropy
network classes gives
\[
  E\|\widehat\Phi_k^{\rm or}-f_{0,k}\|_{k,2}^2
  \leq
  C\inf_{g\in\mathcal F_n}\|g-f_{0,k}\|_{k,2}^2
  +
  C\frac{N_n\log n}{n_k}.
\]
On \(n_k\geq n\mu_0/2\),
\[
  E\|\widehat\Phi_k^{\rm or}-f_{0,k}\|_{k,2}^2
  \leq
  C N_n^{-2\beta/p}
  +
  C\frac{N_n\log n}{n}.
\]
Choosing
\[
  N_n\asymp n^{p/(2\beta+p)}
\]
balances the two terms and yields
\[
  E\|\widehat\Phi_k^{\rm or}-f_{0,k}\|_{k,2}^2
  \lesssim
  n^{-2\beta/(2\beta+p)}
\]
up to logarithmic factors.  The event \(n_k<n\mu_0/2\) is negligible because
its probability is exponentially small and the functions are uniformly
bounded.
\end{proof}

\begin{proof}[Proof of Theorem~\ref{thm:risk}]
Write
\[
  E\{\widehat f(S,X)-f_0(S,X)\}^2
  =
  E_1+E_2,
\]
where
\[
  E_1
  =
  E\left[
    \{\widehat f(S,X)-f_0(S,X)\}^2
    1\{\widetilde C(S)=C_0(S)\}
  \right],
\]
and
\[
  E_2
  =
  E\left[
    \{\widehat f(S,X)-f_0(S,X)\}^2
    1\{\widetilde C(S)\neq C_0(S)\}
  \right].
\]

On the event \(\widetilde C(S)=C_0(S)=k\),
\[
  \widehat f(S,X)-f_0(S,X)
  =
  \widetilde\Phi_k(X)-f_{0,k}(X).
\]
Therefore
\[
\begin{aligned}
  E_1
  &\leq
  \sum_{k=1}^{K_0}
  E\left[
    \{\widetilde\Phi_k(X)-f_{0,k}(X)\}^2
    1\{S\in\mathcal C_k\}
  \right]  \\
  &=
  \sum_{k=1}^{K_0}
  \mu_k E\|\widetilde\Phi_k-f_{0,k}\|_{k,2}^2.
\end{aligned}
\]
The last equality follows by conditioning on the training data and using the
independence of the test point.

For \(E_2\), boundedness gives
\[
  |\widehat f(S,X)-f_0(S,X)|
  \leq
  |\widehat f(S,X)|+|f_0(S,X)|
  \leq 2M.
\]
Thus
\[
  E_2
  \leq
  4M^2P\{\widetilde C(S)\neq C_0(S)\}.
\]
Combining the two bounds proves the theorem.
\end{proof}

\clearpage
\begingroup
\small
\bibliographystyle{plainnat}
\bibliography{references}
\endgroup

\end{document}